\documentclass[smallextended]{svjour3}       % onecolumn (second format)
\usepackage[bottom=4cm, right=4cm, left=4cm, top=4cm]{geometry}
\usepackage{amsfonts}
\usepackage{latexsym}
\usepackage{amssymb}
\usepackage{amsmath}
\usepackage[mathscr]{eucal}
\usepackage{graphicx}
\usepackage{hyperref}
\usepackage{caption}
\usepackage{subcaption}

\usepackage{algorithm}
\usepackage[noend]{algpseudocode}

\renewcommand{\qed}{\hfill{\ \ \rule{2mm}{2mm}} \vspace{0.2in}}

\newcommand{\ind}{1\hspace{-2.3mm}{1}}

\renewcommand{\thefigure}{\arabic{figure}}
\begin{document}

\title{Criticality in Dissimilar Decomposition and Undersampling of Random Datasets with Anomalies}
\titlerunning{Decomposing Random Datasets  with Anomalies}

\author{ \textbf{Ghurumuruhan Ganesan}
\thanks{E-Mail: \texttt{gganesan82@gmail.com} } \\
%EndAName
\ \\
University of Bristol, UK}

%\author{\textbf{Ghurumuruhan Ganesan}}
%\authorrunning{G. Ganesan}
%\institute{IISER, Bhopal\\
%\email{gganesan82@gmail.com }}

%\author{ \textbf{Ghurumuruhan Ganesan}}
%\authorrunning{G. Ganesan}
%\institute{Institute of Mathematical Sciences, HBNI, Chennai\\
%\email{gganesan82@gmail.com }}
\date{}
\maketitle

\begin{abstract}
Training datasets for upcoming LLMs would include a significant amount of AI text/image data generated from current LLMs. In such a scenario, it is important to understand how this affects batch decompositions and thereby, the performance of the resultant new LLM.  In this paper, we consider AI generated data as anomalies ``linked" to main data points and study decomposition and undersampling properties of the overall random dataset.  We use redundancy graphs and iteration techniques to obtain bounds for the minimum size of a strongly dissimilar (SD) decomposition and demonstrate a phase transition phenomena, wherein the minimum size is essentially determined by the \emph{main} data points when the number of anomalies is small  and is ``taken" over by the anomalies above a certain threshold. We also establish a size criticality result for the strong similarity of a randomly undersampled dataset and illustrate  our results with examples involving categorical datasets, whose overall space size is much larger than the size of the dataset.

\vspace{0.1in} \noindent \textbf{Key words:} Random Datasets; Anomalies; Linkage; Strongly Dissimilar Batch Decompositions; Random Undersampling.

\vspace{0.1in} \noindent \textbf{AMS 2000 Subject Classification:} Primary: 60K35, 60J10;
\end{abstract}

\bigskip

\renewcommand{\theequation}{\thesection.\arabic{equation}}
\setcounter{equation}{0}
\section{Introduction} \label{intro}
For better training, large datasets are often split into small batches and fed sequentially to the predictive model, like for example, neural network~\cite{kuhn}. This is done to ensure smooth functioning of the gradient descent algorithms running in the background.  Naturally, it is therefore important that the dataset itself is as nice as possible since quality of the dataset is important in determining the overall performance of the model (see~\cite{taleb},~\cite{gong} and references therein). Recently~\cite{hadi} has studied variable length dataset decomposition in an effort to improve the training time for large language models (LLMs).

An important aspect to consider here is that  each batch in  a decomposition contains as many dissimilar data points as possible, so that the model learns new patterns in each iteration and the learning is balanced.  Recently, this idea was explored in~\cite{ganesan} where bounds were obtained for the  minimum size of batch decomposition of random datasets, with a predetermined degree of similarity per batch. Specifically, similarity was defined via the Euclidean distance between the \emph{continuous} parts of data points and segmentation and martingale based methods were used to derive sufficient conditions  for the \emph{existence} of nice batch decompositions with a given similarity parameter.

Following this,~\cite{ganesan2} studied reduced similarity decompositions of random \emph{categorical} datasets, with a given similarity function. It was argued that even the notion of similarity was fragile in the presence of missing or corrupted entries and  unique batch decompositions (UBDs) were utilized to spread similar data points across multiple batches. Bounds for the minimum size of UBDs with a given relaxation parameter, were obtained using ball and bin techniques.

This paper continues the study of dissimilar decompositions of random datasets, specifically containing anomalies that are deeply linked to the true data points. From the statistical perspective~\cite{abbey}~\cite{agga}, anomalies in datasets have been studied in the form of outliers and methodologies to deal with such outliers are well-known (see for example~\cite{sullivan}~\cite{hidi}~\cite{vin} and references therein). This is important since outliers frequently influence or even reverse the statistical significance and are therefore often subject to intense scrutiny.

An important situation that becomes increasingly relevant in the coming time, is the issue of training future LLMs using data generated from current or previous LLMs: as more and more artificially generated text/image data (AI data) populate the information space everyday, it is clear that training dataset for upcoming LLMs would include a significant amount of AI text/image data as well. Such  anomalies are linked to the true data points via the source and it is important to ensure that the impact of such ``anomalies" is as small as possible. %In this context,  we interpret AI data as anomalies and concern with datasets corrupted by  a small number of anomalies.

Addressing the above issues, in this paper, we study dissimilar decompositions of random datasets  corrupted by anomalies. To quantify the impact of anomalies, we also consider \emph{linkage} between data points, apart from pairwise similarity. We use binomial random graphs and an iterative stable set extraction technique, to establish a phase transition for the minimum size of a strongly dissimilar (SD) decomposition, where each batch contains data points that are dissimilar \emph{and} delinked. We show that if the number of anomalies are small enough, then the minimum size is primarily determined by the ``main" data points.  Conversely, we also demonstrate that minimum size depends crucially on the anomalies above a certain threshold, even if the number of anomalies is much smaller than the overall size of the dataset.

Subsequently, we  use first and second moment methods to estimate the critical size that ensures strong dissimilarity of a randomly undersampled dataset. We illustrate all our results using examples involving categorical datasets, where the overall space has far more elements than the size of the dataset. We remark that we essentially use standard colouring techniques in our proofs and to keep the presentation self-contained, we have avoided any colouring terminology.

The paper is organized as follows: In Section~\ref{sec_main}, we state  our main result (Theorem~\ref{thm_one}) regarding  the phase transition in the minimum size of a dissimilar batch decomposition of a random dataset,  as number of anomalies increase. We  illustrate Theorem~\ref{thm_one} using an example involving categorical datasets with overall space much larger than the size of the dataset.  We also present our second main result, Theorem~\ref{thm_samp}, that obtains bounds for the critical size of a randomly undersampled dataset that is also strongly dissimilar. Following that,  in Section~\ref{sec_pf_thm_one}, we prove Theorem~\ref{thm_one} and in Sections~\ref{sec_pf_thm_samp}-\ref{sec_pf_thm_samp_b}, we establish Theorem~\ref{thm_samp}. %We also describe coroll show that if the number of anomalies is small enough, then the decomposition size is essentially determined by the main data points. We also estimate the critical sampling size that ensures dissimilarity of the randomly undersampled dataset. Following that, in Section~\ref{sec_pf_thm_one} we prove Theorems~\ref{thm_one} and~\ref{thm_samp}, respectively.

For convenience, we have collected the commonly used symbols and their representation, in  Table~\(1\) below.
\begin{table}[h!]
  \begin{center}
    \caption{Notation.}
    \label{tab:table1}
    \begin{tabular}{l|l} % {l|c|r} <-- Alignments: 1st column left, 2nd middle and 3rd right, with vertical lines in between
      \textbf{Symbol} & \textbf{Meaning}\\
      \hline
      \(n\) & \text{size of the overall dataset}\\
      \(v_0\) & \text{number of anomalous data points}\\
      \(A_v\) & \(v^{th}\)~\text{data point}\\
      \(\lambda\) & \text{ similarity threshold}\\
      \(\zeta_{up},\zeta_{low}\) & \text{ extremal similarity probabilities of main data points}\\
      \(p_{up},p_{low}\)  & \text{ extremal similarity probabilities of anomalous data points}\\
      \(q_{up},q_{low}\) & \text{ extremal linkage probabilities of main data points}\\
      \(s_{up},s_{low}\) &  \text{ extremal linkage probabilities of anomalous data points}\\
    \end{tabular}
  \end{center}
\end{table}

\renewcommand{\theequation}{\thesection.\arabic{equation}}
\setcounter{equation}{0}
\section{Main Results} \label{sec_main}
Let~\(\Omega\) be any countable set and for~\(n \geq 1\) let~\(\{A_v\}_{1 \leq v \leq n}\) be independent (but not necessarily identically distributed) elements from~\(\Omega,\) where~\(A_v\) has distribution~\(p_v;\) i.e.,  \[\mathbb{P}(A_v = a) = p_v(a) \text{ for all } a \in \Omega.\]
We refer to~\(A_v\) as the data point with index~\(v\) and~\(\{A_v\}\) as the \emph{dataset}. We also denote~\(\{1,2,\ldots,n\}\) to be the set of indices.

For integer~\(0 \leq v_0 \leq n,\) we assume the data points~\(\{A_v\}_{1 \leq v \leq n-v_0}\) are independent and identically distributed (i.i.d.) with common distribution~\(p(.);\) i.e.,~\(\mathbb{P}(A_v = a) = p(a)\) for any~\(a \in \Omega.\) We refer to~\(A_v\) as the~\(v^{th}\) \emph{main} datapoint.  For~\(n-v_0+1 \leq v \leq n,\) the data point~\(A_v\) has a distribution~\(p_v(.)\) that is different from~\(p(.)\) and we refer to~\(A_v\) as an \emph{anomalous} data point.

Let~\(\rho : \Omega \times \Omega \rightarrow [0,1]\) be any positive real-valued function satisfying:\\
\((a)\)~\(\rho(x,y) =1\) if and only if~\(x=y\) and\\
\((b)\)~\(\rho(x,y) = \rho(y,x)\) for any~\(x, y \in \Omega.\)\\
We define~\(\rho(x,y)\) to be the \emph{similarity} between the data points~\(x\) and~\(y\) and use the definition~\(\rho(x,x) = 1\) to denote that each element is similar to itself.

For example, suppose~\(\Omega = \{0,1\}^{J}\) for some integer~\(J \geq 1\) so that each data point contains~\(J\) features. For data points~\(x = (x_1,\ldots,x_J)\) and~\(y  = (y_1,\ldots,y_J),\) we set
\begin{equation}\label{rho_example}
\rho_0(x,y) := \frac{1}{J}\sum_{j=1}^{J} \ind(x_j = y_j),
\end{equation}
to be the similarity between data points~\(x\) and~\(y,\) where~\(\ind(.)\) refers to the indicator function. In words,~\(J\rho_0\) is simply the number of features where the data points agree.  Clearly~\(\rho_0\) satisfies properties~\((a)-(b)\) above and could therefore be considered as a measure of similarity between data points~\(x\) and~\(y.\)

For~\(\lambda > 0,\) we say that data points~\(A_u\) and~\(A_v\) are~\(\lambda-\)\emph{dissimilar} or simply dissimilar if~\(\rho(A_u,A_v) \leq \lambda.\) We also say that~\(A_u\) and~\(A_v\) are \emph{similar} if~\(\rho(A_u, A_v) > \lambda\)  and let
\begin{equation}\label{zeta_def}
\zeta_{up} := \max_{a \in \Omega} \mathbb{P}\left(\rho(A_1,a)  > \lambda\right) \;\;\text{ and } \;\; \zeta_{low} := \min_{a \in \Omega} \mathbb{P}\left(\rho(A_1,a) > \lambda\right)
\end{equation}
denote the maximum and minimum similarity probabilities between a \emph{main} data point and a categorical element, respectively. Similarly, let \begin{equation}\label{p_def}
p_{up} := \max_{n-v_0+1 \leq v \leq n}\max_{a \in \Omega} \mathbb{P}\left(\rho(A_v,a) > \lambda \right)\;\;\;\text{ and }\;\;\;p_{low} := \min_{n-v_0+1 \leq v \leq n}\min_{a \in \Omega} \mathbb{P}\left(\rho(A_v,a)  > \lambda\right)
\end{equation}
denote the maximum and minimum similarity probabilities, respectively,  between an anomalous data point and an element of~\(\Omega.\)

As mentioned in the introduction, in addition to similarity,  we also utilize the concept of linkage between data points which we define via Binomial random graphs as follows. Let~\(K_n\) be the complete graph with vertex set~\(\{1,2,\ldots, n\}\) and let~\(\{Z(h)\}_{h \in K_n}\) be independent Bernoulli random variables, indexed by the edge set of~\(K_n\) and satisfying
\begin{equation}\label{x_dist}
\mathbb{P}\left(Z(h) = 1\right) = p(h) = 1-\mathbb{P}\left(Z(h) = 0\right)
\end{equation}
where~\(0 < p(h) < 1.\) If~\(h = (u, v)\) has endvertices~\(u\) and~\(v,\) then we also write~\(p(h) = p(u,v) = p(v, u)\) and define~\(p(u,v)\) to be the  \emph{linkage probability} between the data points~\(A_u\) and~\(A_v.\) Similarly~\(Z(h) = Z(u,v)\) is referred to as the \emph{linkage} between~\(A_u\) and~\(A_v,\) with~\(Z(u,v) = 1\) denoting that~\(A_u\) and~\(A_v\) are in fact linked.

We remark here that the linkage between two data points~\(A_u\) and~\(A_v\) does not depend on the actual values of~\(A_u\) or~\(A_v\) but only on the indices~\(u\) and~\(v.\) Effectively, this quantifies linkage between two data points based on their ``source" rather than the actual values and is consistent with our goal to understand the behaviour of anomalies. Indeed, in the AI text/image example described in the introduction, each  anomalous/AI data point might be linked to multiple (or even all) main data points if, in fact, it was  created using real data generated from the same source as~\(\{A_v\}_{1 \leq v \leq n-v_0}.\)

Based on the above discussion, we define
\begin{equation}\label{link_def}
q_{up} := \max_{1 \leq u < v \leq n-v_0} p(u,v) \text{ and } s_{up} := \max_{n-v_0+1 \leq  v \leq n} \max_{1 \leq u \neq v \leq n} p(u,v)
\end{equation}
to be the respective maximal linkage probabilities associated with main and anomalous data points and assume that~\(s_{up} \geq q_{up},\)  to indicate that anomalous data points have higher linkage probability than main data points. We also define the corresponding minimal linkage probabilities
\begin{equation}\label{link_def_low}
q_{low} := \min_{1 \leq u < v \leq n-v_0} p(u,v) \text{ and } s_{low} := \min_{n-v_0+1 \leq  v \leq n} \min_{1 \leq u \neq v \leq n} p(u,v).
\end{equation}

With the above setup, we now extend the definition of similarity described prior to~(\ref{p_def}) and say that data points~\(A_u\) and~\(A_v\) are \emph{strongly similar}, if
\begin{equation}\label{ss_def}
\rho(A_u,A_v) > \lambda \text{ and } Z(u,v) = 0.
\end{equation}
We say that~\(A_u\) and~\(A_v\) are \emph{strongly dissimilar} (SD), if~(\ref{ss_def}) does not hold and we say a subset~\(V \subset \{1,2,\ldots,n\}\) is SD if~\(A_u\) and~\(A_v\) are SD for any~\(u \neq v \in V.\) A decomposition of~\(\{A_v\}_{1 \leq v \leq n}\) of size~\(t\) is a partition~\(\{V_i\}_{1 \leq  i \leq t}\) of the index set~\(\{1,2,\ldots,n\};\) i.e.,
\[V_i \cap V_j = \emptyset\;\;\text{ and }\;\;\bigcup_{i=1}^{t} V_i = \{1,2,\ldots,n\}.\]
\begin{definition}\label{def_ax_ek} We say that the decomposition~\({\cal V} := \{V_i\}_{1 \leq i \leq t}\) is SD if each~\(V_i, 1 \leq i \leq t\) is SD.
\end{definition}
Trivially,~\(V_i = \{i\}, 1 \leq i \leq n,\) is an SD decomposition of the dataset~\(\{A_v\}\) and from a machine learning perspective, it is of interest to obtain dissimilar decompositions that are as small as possible~\cite{ganesan}.

Letting~\(\tau_{SD} = \tau_{SD}(\lambda)\) denote the \emph{minimum} size of a SD decomposition of~\(\{A_v\},\) the following result describes the effect of anomalous data points on~\(\tau_{SD}.\) Throughout, logarithms are natural, constants do not depend on~\(n\) and for two sequences~\(\{a_n\}\) and~\(\{b_n\},\) we use the notations~\(a_n = o(b_n)\) and~\(a_n \gg b_n\) to respectively denote that~\(\frac{a_n}{b_n} \longrightarrow 0\) and~\(\frac{a_n}{b_n} \rightarrow \infty,\) as~\(n \rightarrow \infty.\)
\begin{theorem}\label{thm_one} Let~\(p_{up},p_{low}, \zeta_{up}\) and~\(\zeta_{low}\) be as in~(\ref{p_def})-(\ref{zeta_def}) and assume that the linkage probabilities defined in~(\ref{link_def}) and~(\ref{link_def_low}) satisfy~\(s_{up} \geq q_{up} \geq q_{low}.\)\\
\((a)\)  For every~\(\varepsilon > 0,\) there is a constant~\(M > 0\) such that
\begin{equation}\label{sfd_up}
\mathbb{P}\left(\tau_{SD} \leq (1+\varepsilon)\min\left(\Lambda, (\Sigma + \sqrt{v_0\Lambda})\log{n} \right)\right) = 1-o(1),
\end{equation}
where \begin{equation}\label{sigma_def}
\Sigma  :=  n\zeta_{up}q_{up} + v_0p_{up}s_{up} + M\log{n} \;\;\text{ and }\;\; \Lambda :=  (n\zeta_{up} + v_0p_{up})s_{up} + M\log{n}
\end{equation}
\((b)\) Conversely, if~\(v_0 = o(n)\) and~\(\Sigma_{low} := n\zeta_{low}q_{low} + v_0p_{low}s_{low} \gg \log{n},\) then
\begin{equation}\label{sfd_low}
\mathbb{P}\left(\tau_{SD} \geq  \frac{\Sigma_{low}}{8\log{n}}\right) = 1-o(1).
\end{equation}
\((c)\) In particular if~\(v_{low} := \min\left(n \zeta_{up}q_{up}^2, \frac{n\zeta_{low}q_{low}}{p_{low}s_{low}}\right), v_0 = o(v_{low})\) and~\( n\zeta_{low}q_{low} \gg \log{n},\) then
\begin{equation}\label{sfd_sub}
\mathbb{P}\left(\frac{\gamma_1 n \zeta_{low}q_{low}}{\log{n}} \leq \tau_{SD} \leq \gamma_2 n\zeta_{up}\min(s_{up},q_{up} \log{n}) \right) = 1-o(1),
\end{equation}
for some constants~\(\gamma_1,\gamma_2 > 0.\) Conversely, if~\(v_{up} := \max\left(\frac{n\zeta_{up}}{p_{up}},\frac{n\zeta_{low}q_{low}}{p_{low}s_{low}}\right), n \gg v_0 \gg v_{up}\) and~\( v_0p_{low}s_{low} \gg \log{n},\)
\begin{equation}\label{sfd_sup}
\mathbb{P}\left( \frac{\theta_1 v_0p_{low}s_{low}}{\log{n}} \leq \tau_{SD} \leq \theta_2 v_0  p_{up}s_{up} \right) = 1-o(1),
\end{equation}
for some constants~\(\theta_1,\theta_2 > 0.\)
\end{theorem}
Essentially, we see that the minimum decomposition size~\(\tau_{SD}\) undergoes a phase transition with regards to the number of anomalous data points~\(v_0:\) if the dataset is ``nearly" homogenous and~\(v_0\) is small, then~\(\tau_{SD}\)  is essentially determined  by the main data points. Conversely, if~\(v_0\) is large enough, then~\(\tau_{SD}\) in fact crucially depends on the anomalies, \emph{even if}~\(v_0\) is negligible compared to the size of the dataset~\(n.\) This highlights the effect of the strong links between anomalies and the main data points.

%and define the auxiliary probabilities
%\begin{equation}\label{alpha_def}
%\alpha_{up} := \max\left(p_{up}, \frac{\log{n}}{n}\right) \text{ and }\theta_{up} := \max\left(\zeta_{up}, \frac{\log{n}}{n}\right).
%\end{equation}

A couple of remarks:\\
\underline{\emph{Remark 1}}: While the bounds in Theorem~\ref{thm_one} are in terms of extremal probabilities, it would be interesting to obtain estimates in terms of the average similarity probabilities.\\
\underline{\emph{Remark 2}}: In the context of datasets containing AI data, Theorem~\ref{thm_one} also hints at the possibility of a singularity or skynet threshold above which the AI data essentially take over.

We illustrate the bounds in Theorem~\ref{thm_one} with the following example. For integers~\(L=L(n) \geq r=r(n) \geq 2,\) let~\(\Omega := \{a_1,\ldots, a_{L}\}\) be any set containing~\(L\) elements and let~\(H_{reg}\) be an~\((r-1)-\)regular subgraph with vertex set~\(\Omega\) (such graphs are well-known and for illustration, see also~\cite{ganesan_dam}). We assume that the graph~\(H_{reg}\) describes the similarity between various elements of~\(\Omega;\) i.e., an edge~\((x,y)\) with endvertices~\(x\) and~\(y\) is present in~\(H_{reg}\) if and only if~\(\rho(x,y) \geq \lambda,\) where~\(0 < \lambda < 1.\) We use the convention that~\(\rho(x,x) = 1\) so that each element is similar to itself.

Let~\(p(.)\) be the uniform distribution with support~\(\Omega\) satisfying~\(p(a) = \frac{1}{L}\) for all~\(a \in \Omega.\) We assume that all main data points have distribution~\(p.\) Similarly, for~\( 0 < b = b(n) < 1,\) the anomalous data point distribution is defined as
\begin{equation}\label{out_dist}
p_e(a) :=
\left\{
\begin{array}{ll}
\frac{1-b}{L}, & \text{ if } a \in \Omega \setminus \{a_{out}\}\\
&\\
b + \frac{1-b}{L}, & \text{ if } a = a_{out}.
\end{array}
\right.
\end{equation}
Since~\(b > 0\) strictly, the distribution~\(p_e(.)\) ``prefers" the element~\(a_{out}\) over other elements of~\(\Omega.\)

We assume that the number of anomalous data points~\(v_0 \geq 1\) and so the~\(n^{th}\) data point~\(A_n\) is in fact an anomalous data point. Also any two main data points are linked with probability~\(q = o(1),\) while any anomalous data point is linked to any other data point with probability~\(1;\) i.e.,
\begin{equation}\label{link_example}
q_{low} = q_{up} = q\;\;\text{ and }\;\;s_{up} = s_{low} =  1.
\end{equation}

The following is a direct corollary of Theorem~\ref{thm_one} and is proved along with Theorem~\ref{thm_one}, in Section~\ref{sec_pf_thm_one}.
\begin{corollary} \label{cor_one}  Suppose~(\ref{link_example}) holds and
\begin{equation}\label{param_choice}
\frac{1}{\log{n}} \gg b \gg \frac{r}{L} \gg \frac{\log{n}}{nq}.
\end{equation}
If~\(v_0 = o\left(\frac{nrq^2}{L}\right)\) then there are constants~\(\gamma_1,\gamma_2 > 0\) such that
\begin{equation}\label{sfd_sub_example}
\mathbb{P}\left(\frac{\gamma_1 nr q}{L \log{n}} \leq \tau_{SD} \leq \frac{\gamma_2 nr q\log{n}}{L} \right) = 1-o(1).
\end{equation}
Conversely, if~\(n \gg v_0 \gg \frac{nr}{bL},\) then there are constants~\(\theta_1,\theta_2 > 0\) such that
\begin{equation}\label{sfd_sup_example}
\mathbb{P}\left(\frac{\theta_1 v_0b}{\log{n}} \leq \tau_{SD} \leq \theta_2 v_0 b  \right) = 1-o(1).
\end{equation}
\end{corollary}
The condition~\(b \gg \frac{r}{L}\) implies that anomalies are ``more" similar to~\(a_{out}\) than main data points and the relations in~(\ref{param_choice}) are satisfied, for example, if~\[r = n^{\alpha},\;\;b = \frac{1}{n^{\beta}},\;\;L = n^{1+\theta}\;\; \text{ and }\;\;q = \frac{1}{n^{\delta}}\] for some constants~\(\alpha,\theta,\beta,\delta>0\) satisfying \[0 < \beta < 1+\theta-\alpha < 1-\delta.\]

We see that if the number of anomalous data points~\(v_0\) is small, then with high probability (i.e., with probability~\(1-o(1)\)), the minimum size of a dissimilar decomposition is essentially determined by main data points and is of the order of~\(\frac{nr}{L},\) modulo logarithmic factors. On the other hand if~\(v_0 = o(n)\) but is much larger than~\(\frac{nr}{bL}\) (this is consistent since~\(b \gg \frac{r}{L}\) by~(\ref{param_choice})), then with high probability~\(\tau_{SD}\) is roughly of the order of~\(v_0b,\)  modulo logarithmic factors.

\subsection*{\em Random Undersampling}
Often in applications involving classification, it is important to undersample a dataset to a predetermined size, especially if the size of one of the classes is relatively large compared to the rest of the classes~\cite{kuhn}. In what follows, we therefore study the dissimilarity properties of the undersampled dataset,  in the presence of anomalies.

As before, let~\(\{A_v\}_{1 \leq v \leq n-v_0}\) be the main data points that are i.i.d.\ and set~\(\{A_v\}_{n-v_0+1 \leq v \leq n}\) be the outlier data points with distribution distinct from the main data points. Because we may not be able to or may not be interested in identifying the anomalies, we perform a \emph{random} undersampling of the dataset and obtain a data subset of size~\(1 \leq t \leq n\) as follows. Pick one index~\(Z_1\) uniformly randomly from~\({\cal A}_0 := \{1,\ldots,n\}\) and set~\({\cal A}_1 := {\cal A}_0 \setminus \{Z_1\}.\) Iteratively for~\(1 \leq w \leq t-1,\)  pick~\(Z_{w+1}\) uniformly randomly from~\({\cal A}_w\) and set~\({\cal A}_{w+1} := {\cal A}_w \setminus \{Z_{w+1}\}.\) We define~\(\{Z_w\}_{1 \leq w \leq t}\) to be the indices of the \emph{randomly undersampled} (RUS) data subset of size~\(t.\)

Intuitively, if the number of anomalies is small and the size of the sampled dataset is also small, then we expect that the sampled data points are strongly dissimilar (SD), i.e., dissimilar and delinked, with high probability. As the sample size increases, this becomes less and less likely. We formalize this intuition in the following result.  Letting~\(E_{SD}(t)\) be the event that the undersampled dataset~\(\{A_{Z_k}\}_{1 \leq k \leq t}\) is SD, we have the following transition result.
\begin{theorem}\label{thm_samp} Let~\(\zeta_{up}\) and~\(\zeta_{low}\) be the extremal similarity probabilities as in~(\ref{zeta_def}) and let~\(q_{up}\) and~\(q_{low}\) be the linkage probabilities, as defined in~(\ref{link_def}) and~(\ref{link_def_low}), respectively. Set~\(v_{samp} := n\zeta_{up}q^2_{up}\) and suppose  that the number of anomalies~\(v_0 = o(v_{samp}).\)\\
\((a)\) If~\(t_{low} := \sqrt{\frac{n}{v_{samp}+\log{n}}}\) and~\(t  = o(t_{low}),\) then
\[\mathbb{P}\left(E_{SD}(t)\right) =  1-o(1).\]
\((b)\) Conversely, if~\(n\zeta_{low}q_{low} \gg \log{n} \) and
\begin{equation}\label{fairuza}
n \gg t \gg t_{up} := \max\left(\frac{1}{\zeta_{up}q_{up}}, \frac{\zeta_{up}^2q_{up}^2}{\zeta^2_{low}q^2_{low}}, \frac{1}{\sqrt{\zeta_{low}q_{low}}}\right),
\end{equation} then
\[\mathbb{P}\left(E_{SD}(t)\right) =  o(1).\]
\end{theorem}
For datasets with small number of anomalies, the above result essentially estimates the ``critical" size that ensures strong dissimilarity of the undersampled dataset, with high probability.

As before, for illustration,  we consider the example described in Corollary~\ref{cor_one}.
\begin{corollary} \label{cor_two}  Suppose~(\ref{link_example}) and~(\ref{param_choice}) holds and assume that the number of anomalies~\(v_0 = o\left(\frac{nrq^2}{L}\right).\) If~\(t = o\left(\frac{1}{q}\sqrt{\frac{L}{r}}\right),\) then~\(\mathbb{P}(E_{SD}(t)) = 1-o(1)\) and if~\(t \gg \frac{L}{qr},\) then~\(\mathbb{P}(E_{SD}(t)) = o(1).\)
\end{corollary}

\renewcommand{\theequation}{\thesection.\arabic{equation}}
\setcounter{equation}{0}
\section{Proof of Theorem~\ref{thm_one} and Corollary~\ref{cor_one}} \label{sec_pf_thm_one}
For convenience, we have compartmentalized the bulk of our proof  into a couple of auxiliary Lemmas, that are also of independent interest. We begin with the construction of the similarity graph~\(\Gamma\) corresponding to the overall dataset~\(\{A_v\}\) containing~\(n\) data points. Let~\(K_n\) be the complete graph with vertex set~\(\{1,2,\ldots,n\}\) and let~\(\Gamma  = \Gamma(\lambda) \subset K_n\) be the graph obtained by connecting any two vertices~\(u\) and~\(v\)  by an edge if and only if the corresponding data points~\(A_u\) and~\(A_v\) are strongly similar; i.e.,~\(\rho(A_u, A_v) > \lambda\) and~\(Z(u,v) = 1.\) We define~\(\Gamma\) to be the \emph{similarity graph} of the dataset~\(\{A_v\}.\)

Lemma~\ref{lemma_max_ave_deg} below describes a deterministic bound that estimates the minimum size~\(\tau_{SD}\) of a dissimilar decomposition, in terms of the maximum vertex degree and \emph{maximum average} degree of~\(\Gamma.\) Next, in Lemma~\ref{lem_e_main}, we evaluate the  maximum vertex degree and maximum average degree in terms of the similarity and linkage probabilities defined in~(\ref{p_def}),~(\ref{zeta_def}) and~(\ref{link_def}). Finally, we combine these ingredients to establish the bounds in Theorem~\ref{thm_one}.

For~\(1 \leq v \leq n,\) let
\begin{equation}\label{cv_def}
{\cal C}(v) := \{1 \leq u  \neq v \leq n : \rho(A_u, A_v) > \lambda \text{ and } Z(u,v) = 1\}
\end{equation}
be the set of all indices of the data points that are similar to~\(A_v.\) Clearly,~\(d(v) := \#{\cal C}(v)\) is the degree of the vertex~\(v\) in~\(\Gamma\) and we denote
\begin{equation}\label{delta_def}
\Delta = \Delta(n) := \max_{1 \leq v \leq n}d(v)
\end{equation}
to be the maximum size of a similarity set.

The average vertex degree in any subgraph~\(H \subseteq \Gamma\) is defined as
\begin{equation}\label{ave_deg_def}
d_{av}(H) := \frac{1}{n}\sum_{v=1}^{n} d(v,H)
\end{equation}
where~\(d(v,H)\) is the degree of vertex~\(v\) in~\(H\) and we refer to the average vertex degree in~\(\Gamma\) simply as~\(d_{av}.\) Setting
\begin{equation}\label{max_ave_def}
h_{av} = h_{av}(\Gamma) := \max_{H} d_{av}(H),
\end{equation}
to be the \emph{maximum average vertex degree}~\cite{west},  where the maximum is over all subgraphs~\(H \subseteq \Gamma,\) we have the following result.
\begin{lemma}\label{lemma_max_ave_deg} We have that
\begin{equation}\label{chi_bound}
\tau_{SD} \leq \min\left(\Delta+1, 2\chi_{av}\right),
\end{equation}
where~\[\chi_{av} := \max(1,h_{av})\cdot \log\left(\frac{ne}{\max(1,h_{av})}\right).\]
\end{lemma}
The estimate in~(\ref{chi_bound}) is applicable for both sparse graphs that have small~\(\Delta\) and for graphs that have high~\(\Delta\) but small~\(h_{av}.\) For example, the path on~\(n\) vertices has~\(\Delta =2\) and the star graph on~\(n\) vertices has~\(\Delta = n-1\) but~\(h_{av} = 1.5.\)

\emph{Proof of Lemma~\ref{lemma_max_ave_deg}}: A standard greedy assignment gives us a dissimilar batch decomposition of size~\(\Delta+1\) and for completeness, we give a small proof. Initially let~\({\cal V}_{l}, 1 \leq l \leq \Delta+1,\) be~\(\Delta+1\) empty batches. In the first step, pick data point~\(A_1\) and assign it to batch~\({\cal V}_1.\) Inductively, at the~\(j^{th}\) step, the data points~\(\{A_l\}_{1 \leq l \leq j-1}\) having been already assigned, we seek a batch for~\(A_j.\) By definition, there are at most~\(\Delta\) data points similar to~\(A_j\) and since there are~\(\Delta+1\) batches, there is always at least one batch (call it~\({\cal V}_{z(j)}\)) that does not contain any data point similar to~\(A_{j}.\) Assign~\(A_{j}\) to~\({\cal V}_{z(j)}.\) Continue this way until all data points are exhausted. The resulting decomposition containing contains~\(\Delta+1\) batches and is dissimilar. Thus~\(\tau_{SD} \leq \Delta +1.\)

In what follows, we use an iterative batch assignment procedure to establish the~\(\chi_{av}\) bound in~(\ref{chi_bound}). A set of vertices~\({\cal S}\) is said to be \emph{stable} in~\(\Gamma\) if no edge of~\(\Gamma\) contains both its endvertices in~\({\cal S}.\) Thus if~\({\cal S}\) is stable, then the corresponding data points~\(\{A_v\}_{v \in {\cal S}}\) form a dissimilar subset.  Let~\({\cal I}_1 = \{v_1,\ldots,v_t\}\) be a maximum stable set, i.e., a stable set of maximum size, in~\(G_1 := \Gamma.\) We assign all vertices in~\({\cal I}_1\) to the first batch and remove all the vertices of~\({\cal I}_1\) from~\(G_1\) to obtain a graph~\(G_2.\) We now repeat the above procedure with~\(G_2.\) Letting~\({\cal I}_2\) be the maximum stable set in~\(G_2,\) we assign  all vertices in~\({\cal I}_2\) to the second batch.

Continuing the above procedure, let~\(G_{k+1}\) be the graph obtained at the end of~\(k\) iterations and suppose~\(G_{k+1}\) contains~\(n_{k+1}\) vertices. By construction, we have created~\(k\) batches so far and we now assign each of the~\(n_{k+1}\) vertices of~\(G_{k+1}\) to a new distinct  batch to get that
\begin{equation}\label{chi_g_k}
\tau_{SD} \leq k + n_{k+1}.
\end{equation}
To estimate~\(n_{k+1},\) we bound the size of~\({\cal I}_i\) for each~\(1 \leq i  \leq k.\) Recalling that~\(d_{av}\) and~\(h_{av} \geq d_{av}\) are the average degree and maximum average degree of~\(\Gamma,\) respectively, we get from Theorem~\(3.2.1,\) pp.~\(29\)~\cite{alon} that~\({\cal I}_1\)  has cardinality
\begin{equation}\label{alp_def}
\#{\cal I}_1 \geq  \frac{n}{2\max(1,d_{av})} \geq \frac{n}{2\max(1,h_{av})} := n\alpha.
\end{equation}
Consequently, the graph~\(G_2\) has~\(n_2 \leq n(1-\alpha)\) vertices and a maximum average degree of at most~\(h_{av}\) and so again arguing as in~(\ref{alp_def}) we get that~\(\#{\cal I}_2 \geq n_2 \alpha.\) This in turn implies that~\(G_3\) has~\(n_3 \leq n_2(1-\alpha) \leq n(1-\alpha)^2\) vertices.

Proceeding iteratively, we get that~\(G_{k+1}\) has~\(n_{k+1} \leq n(1-\alpha)^{k}\) vertices and  so from~(\ref{chi_g_k}), we have~\[\tau_{SD} \leq k + ne^{-\theta k} =: s(k)\] where~\(\theta := |\log(1-\alpha)|.\) The term~\(s(k)\) is minimized if~\(k\) satisfies~\(1- n\theta \cdot e^{-\theta k} = 0\) and for this value of~\(k,\) we get that
\begin{equation}\label{gelat}
\tau_{SD} \leq \frac{\log(n\theta)+1}{\theta}.
\end{equation}
Using~\(|\log(1-x)| > x\) we see that~\(\theta > \alpha\) and since~\(\alpha = \frac{1}{2\max(1,h_{av})} \leq \frac{1}{2},\) we also obtain~\[\theta  = | \log(1-\alpha)| \leq \sum_{j \geq 1} \alpha^{j} = \frac{\alpha}{1-\alpha} \leq 2\alpha.\] Thus~\(\alpha \leq \theta \leq 2\alpha\)
and substituting these bounds into~(\ref{gelat}), we get
\begin{align}
\tau_{SD}  &\leq \frac{\log(2n\alpha)+1}{\alpha} \nonumber\\
&= 2\max(1,h_{av})\cdot \left(\log\left(\frac{n}{\max(1,h_{av})}\right) + 1\right), \nonumber
\end{align}
thus proving~(\ref{chi_bound}). This completes the proof of the Lemma.~\(\qed\)

The following  Lemma describes an  event that facilitates the evaluation of the terms~\(\Delta\) and~\(\chi_{av},\) respectively, occurring in Lemma~\ref{lemma_max_ave_deg}.  As before, let~\(d(v) = \#{\cal C}(v)\) be the number of data points similar to the~\(v^{th}\) data point~\(A_v\) (which is also the degree of vertex~\(v\) in~\(\Gamma\)). Recalling that~\(A_v, 1 \leq v \leq n-v_0\) are main data points, let~\(\zeta_{up}\) and~\(p_{up}\) be the maximum main and anomalous data point similarity probabilities, as defined in~(\ref{zeta_def}) and~(\ref{p_def}), respectively. For~\(M > 0\) let~\(\Sigma = \Sigma(n,M)\) and~\(\Lambda = \Lambda(n,M)\) be as in~(\ref{sigma_def}) and define
\begin{equation}\label{e_sim_def}
E_{sim}  = E_{sim}(\varepsilon, M) := \bigcap_{v=1}^{n-v_0} \left\{d(v) \leq (1+\varepsilon)\Sigma\right\} \bigcap \bigcap_{v=n-v_0+1}^{n} \left\{d(v) \leq (1+\varepsilon)\Lambda\right\}
\end{equation} to be the event that the maximum size of the similarity set of any main data point is at most~\((1+\varepsilon)\Sigma\) and the maximum size of the similarity set of any anomalous data point is at most~\((1+\varepsilon)\Lambda.\) We have the following.
\begin{lemma}\label{lem_e_main} For every~\(\varepsilon > 0,\) there is a constant~\(M > 0\) such that
\begin{equation}
\mathbb{P}\left(E_{sim}\right) \geq 1-\frac{1}{n^3} = 1-o(1).\label{e_sim_est}
\end{equation}
\end{lemma}
In the next Lemma, we use the event~\(E_{sim}\) to estimate the term~\(\chi_{av}\) in Lemma~\ref{lemma_max_ave_deg}.

Throughout, we use the following standard deviation estimate regarding sums of independent Bernoulli random variables.  Let~\(\{X_j\}_{1 \leq j \leq r}\) be independent Bernoulli random variables with~\[\mathbb{P}(X_j = 1) = 1-\mathbb{P}(X_j = 0) > 0.\] If~\(T_r := \sum_{j=1}^{r} X_j\) and~\(0 < \gamma \leq \frac{1}{2},\) then
\begin{equation}\label{conc_est_f}
\mathbb{P}\left(\left|T_r - \mathbb{E}T_r\right| \geq  \gamma \mathbb{E}T_r \right) \leq 2\exp\left(-\frac{\gamma^2}{4}\mathbb{E}T_r\right)
\end{equation}
for all \(r \geq 1.\) For a proof of~(\ref{conc_est_f}), we refer to Corollary A.1.14, pp. 312 of~\cite{alon}.

\emph{Proof of Lemma~\ref{lem_e_main}}: Suppose~\(1 \leq v \leq n-v_0\) so that~\(A_v\) is a main data point. From the definition of extremal probabilities in~(\ref{p_def}) and the linkage probabilities in~(\ref{link_def}), we see that the expected number of main data points similar to~\(A_v\) is at most~\((n-v_0)\zeta_{up}q_{up} \leq n\zeta_{up}q_{up}\) and the expected number of anomalous data points similar to~\(A_v\) is at most~\(v_0p_{up}s_{up}.\) Extending  this argument,  we also have that~\(d(v)\) is stochastically dominated from above by a sum of two independent random variables~\(Z_1(v)+Z_2(v),\) where~\(Z_1(v)\) is Binomially distributed with parameters~\(n-v_0\) and~\(\zeta_{up}q \leq \max\left(\zeta_{up}q_{up} ,\frac{M\log{n}}{2(n-v_0)}\right)\) and~\(Z_2(v)\) is Binomial with parameters~\(v_0\) and~\[p_{up}s_{up} \leq \max\left(p_{up}s_{up}, \frac{M \log{n}}{2v_0}\right),\] where~\(M > 0\) is a constant to be determined later.

Given~\(\varepsilon  >0,\) we therefore get from the deviation estimate~(\ref{conc_est_f}) that
\begin{align}\label{z_one_est}
&\mathbb{P}\left(Z_1(v) \geq (1+\varepsilon)\max\left(n\zeta_{up}q_{up}, \frac{M\log{n}}{2}\right)\right)  \nonumber\\
&\;\;\;\;\leq \mathbb{P}\left(Z_1(v) \geq (1+\varepsilon)(n-v_0)\max\left(\zeta_{up}q_{up}, \frac{M\log{n}}{2(n-v_0)}\right)\right)\nonumber\\
&\;\;\;\;\leq \exp\left(-C(n-v_0)\max\left(\zeta_{up}q_{up}, \frac{M\log{n}}{2(n-v_0)}\right)\right) \nonumber\\
&\;\;\;\;\leq e^{-CM\log{n}}
\end{align}
and similarly,
\begin{equation}\label{z_two_est}
\mathbb{P}\left(Z_2(v) \geq (1+\varepsilon)\max\left(v_0p_{up}s_{up}, \frac{M \log{n}}{2}\right) \right) \leq e^{-CM \log{n}}
\end{equation}
for some constant~\(C  = C(\varepsilon) > 0\) not depending on the choice of~\(M.\)

Choosing~\(M > 0\) large enough, we get from~(\ref{z_one_est}),~(\ref{z_two_est}) and the union bound that
\begin{align}
Z_1(v) + Z_2(v) &\leq (1+\varepsilon)\max\left(n\zeta_{up}q_{up}, \frac{M\log{n}}{2}\right) + (1+\varepsilon)\max\left(v_0p_{up}s_{up}, \frac{M \log{n}}{2}\right) \nonumber\\
&\leq (1+\varepsilon)\Sigma, \nonumber
\end{align}
with probability at least~\(1-\frac{2}{n^{4}},\) where~\(\Sigma =  n\zeta_{up}q_{up} + v_0p_{up}s_{up} + M\log{n}\) is as in Theorem~\ref{thm_one} statement.  Since~\(d(v)\) is stochastically dominated from above by the sum~\(Z_1(v) + Z_2(v),\) we get
\begin{equation}\label{dv_main}
\mathbb{P}\left(d(v) \leq (1+\varepsilon)\Sigma\right) \geq 1- \frac{2}{n^4}
\end{equation}
for each main data point index~\(1 \leq v \leq n-v_0.\)

Similarly, arguing for the case when~\(n-v_0 + 1 \leq v \leq n,\)  we get that
\begin{equation}\label{dv_anom}
\mathbb{P}\left(d(v) \leq (1+\varepsilon)\Lambda\right) \geq 1- \frac{2}{n^4}
\end{equation}
for each anomalous data point, where~\(\Lambda =  (n\zeta_{up} + v_0p_{up})s_{up} + 2M\log{n}\) is again as in Theorem~\ref{thm_one} statement. Finally, combining~(\ref{dv_anom}) and~(\ref{dv_main}) and applying  the union bound, we obtain the desired bound~(\ref{e_sim_est}). This completes the proof of the Lemma.~\(\qed\)

%Recalling that~\(v_0p_{up} = o(n\zeta_{up})\) by Lemma statement, we see that the term~\(\max\left(v_0p_{up},M \log{n}\right)\) is much smaller than~\(\max\left(n\zeta_{up}, M\log{n}\right)= \Sigma.\)

Our final ingredient  uses the  event~\(E_{sim}\)  defined in Lemma~\ref{lem_e_main} to estimate the terms~\(\Delta\) and~\(\chi_{av}\) occurring in the bound~(\ref{chi_bound}).
\begin{lemma}\label{lemma_fin} For constant~\(M > 0,\) let~\(\Sigma  = \Sigma(n,M)\) and~\(\Lambda = \Lambda(n,M)\) be as defined in~(\ref{sigma_def}). Also assume that the anomalous linkage probability~\(s_{up} \geq q\) as in Theorem~\ref{thm_one} statement. If~\(E_{sim}\)  occurs, then
\begin{equation}\label{chi_av_fin}
\Delta \leq (1+\varepsilon) \Lambda \;\; \text{ and }\;\;\chi_{av} \leq (1+2\varepsilon)\left(\Sigma + \sqrt{v_0\Lambda}\right) \log{n},
\end{equation}
for all~\(n\) large.
\end{lemma}
We remark that the bound~(\ref{chi_av_fin}) holds \emph{irrespective} of the value of the constant~\(M.\) This allows us invoke Lemma~\ref{lem_e_main} and choose~\(M\) large enough so that~\(E_{sim}\) occurs with high probability; details are described later in our proof of Theorem~\ref{thm_one}.

\emph{Proof of Lemma~\ref{lemma_fin}}: Since~\(s_{up} \geq q_{up},\) we get that~\(\Sigma \leq \Lambda\) (see~(\ref{sigma_def})) and so the occurrence of the event~\(E_{sim}\) defined in~(\ref{e_sim_def}) implies that~\(\Delta \leq (1+\varepsilon) \Lambda.\) From the definition of~\(\chi_{av}\) in~(\ref{chi_bound}), we get  that
\begin{equation}
\chi_{av} \leq (1+h_{av})\cdot \left(\log{n}+1\right) = (1+h_{av})\log{n}(1+o(1)). \label{chi_bound_ax}
\end{equation}
In what follows, we derive a ``convenient" expression for~\(h_{av}\) and then use the event~\(E_{sim}\) to upper bound~\(h_{av}.\)

For any integer~\(1 \leq T \leq n,\) we have that
\begin{eqnarray}\label{max_ave_2}
h_{av} &=& \max\left(\max_{1 \leq l \leq T} \max_{H_l} d_{av}(H_l), \max_{T+1 \leq l \leq n} \max_{H_l} d_{av}(H_l)\right) \nonumber\\
&\leq& \max\left(T, \max_{T+1 \leq l \leq n} \max_{H_l} d_{av}(H_l)\right) \nonumber\\
&=& \max\left(T,\max_{T+1 \leq l \leq n} \max_{H_l} \frac{2m(H_l)}{l}\right)
\end{eqnarray}
where~\(H_l\) is an \emph{induced} subgraph of the similarity graph~\(\Gamma\) (see discussion prior to Lemma~\ref{lemma_max_ave_deg}), containing~\(l\) vertices and~\(m(H_l)\) is the number of edges in~\(H_l.\)

The event~\(E_{sim}\)  guarantees that at most~\(v_0\) vertices in~\(\Gamma\) have degree larger than
\begin{equation}\label{q_one_def}
Q_1 = Q_1(n) := (1+\varepsilon)\Sigma
\end{equation} and   that  every vertex has degree at most
\begin{equation}\label{q_two_def}
Q_2 = Q_2(n) := (1+\varepsilon)\Lambda.
\end{equation} Therefore using the handshaking relation that the sum of vertex degrees is twice the number of edges in any graph, we obtain \[2m(H_l) \leq lQ_1 + v_0Q_2\] and plugging this into~(\ref{max_ave_2}), we get
\begin{equation}
h_{av} \leq \max\left(T,  \max_{T+1 \leq l \leq n} Q_1 + \frac{v_0Q_2}{l}\right) \leq \max\left(T,Q_1 + \frac{v_0Q_2}{T}\right). \label{max_ave_3}
\end{equation}
Solving for~\(T = Q_1+ \frac{v_0Q_2}{T}\) in~(\ref{max_ave_3}) we get that
\begin{equation} \nonumber
h_{av} \leq T = \frac{Q_1 + \sqrt{Q_1^2+4v_0Q_2}}{2} \leq Q_1 + \sqrt{v_0Q_2},
\end{equation}
using~\(\sqrt{a+b} \leq \sqrt{a} + \sqrt{b}\) for~\(a,b> 0.\) Plugging this into~(\ref{chi_bound_ax}), we obtain
\[\chi_{av} \leq (1+\varepsilon)\left(\Sigma + \sqrt{v_0\Lambda}\right)(1+o(1)) \log{n}\] for all~\(n\) large. This obtains~(\ref{chi_av_fin}) and therefore completes the proof of the Lemma.~\(\qed\)

%Recalling that~\(Q_1 \geq \Sigma = \max(n\zeta_{up}, M\log{n})\) (see~(\ref{q_one_def})) we choose the constant~\(M > 1\) so that~\[Q_1 \geq \max(n\zeta_{up}, \log{n}) = n\theta_{up},\] where~\(\theta_{up} = \max\left(\zeta_{up},\frac{\log{n}}{n}\right) \) is as defined in~(\ref{alpha_def}). Further recalling the condition~\(v_0  = o(v_{low})\) in Lemma statement, we get that~\(v_0\alpha_{up} = o(n\theta^2_{up})\) and so from the definition of~\(Q_2\) in~(\ref{q_two_def}) we get that~\[v_0Q_2 = o(Q_1^2).\] Consequently, again using the expression for~\(Q_1\) in~(\ref{q_one_def}) we get that \[h_{av} \leq Q_1(1+o(1)) = (1+2\varepsilon)\Sigma(1+o(1))\] for all~\(n\) large.

Using the above Lemma, we now prove Theorem~\ref{thm_one}.\\
\emph{Proof of Theorem~\ref{thm_one}}:  We prove parts~\((a), (b)\) and~\((c)\) in that order below. Choose constant~\( M > 0\) large enough and invoke Lemma~\ref{lem_e_main}, so that the event~\(E_{sim}\) occurs with high probability.

From Lemma~\ref{lemma_fin} we get  that~\(\Delta \leq (1+\varepsilon)\Sigma\) and so~\(\Delta + 1 \leq (1+2\varepsilon)\Sigma\) for all~\(n\) large, since~\(\Sigma\) is at least of the order of~\(\log{n}\) as defined in~(\ref{sigma_def}). Similarly, the term~\[\chi_{av} \leq (1+2\varepsilon)(\Sigma+\sqrt{v_0\Lambda}) \log{n}\] as well. Plugging these bounds into~(\ref{chi_bound}) and recalling that~\(\varepsilon > 0\) is arbitrary, we obtain the desired bound~(\ref{sfd_up}) for~\(\tau_{SD}\) in Theorem statement. This completes the proof of part~\((a)\) of the Theorem.

For the lower bound for~\(\tau_{SD},\) we argue as follows. Suppose first that~\(n\zeta_{low}q_{low} \geq v_0p_{low}s_{low}\) and that~\(n\zeta_{low}q_{low} \gg \log{n}.\) Fix~\(0 < \varepsilon <1\) and pick any element~\(a \in \Omega.\) Since the number of anomalies~\(v_0 = o(n),\) each of the~\(n-v_0 \geq n(1-\varepsilon)\) main data points are independently similar to~\(a,\) with probability at least~\(\zeta_{low}.\)  Consequently, if~\({\cal S}(a) \subset \{1,2,\ldots,n-v_0\}\) is the set of indices of data points that are  similar to~\(a,\) then the size~\(N_a = \#{\cal S}(a)\) of the set~\({\cal S}(a)\) is stochastically dominated from below by a Binomial random variable with parameters~\(n(1-\varepsilon)\) and~\(\zeta_{low}.\) Using the deviation estimate~(\ref{conc_est_f}), we then get that
\begin{equation}\label{sa_est}
\mathbb{P}\left(N_a \geq n(1-\varepsilon)^2 \zeta_{low}\right) \geq 1-\exp\left(-Cn\zeta_{low}\right) = 1-o(1),
\end{equation}
for some constant~\(C > 0\) not depending on the choice of~\(a,\) and the final relation in~(\ref{sa_est}) is true since~\(n\zeta_{low} \gg \log{n}\) by Theorem statement.

Suppose the event~\[ E_a := \{N_a \geq n(1-\varepsilon)^2 \zeta_{low}\}\] and let~\(T_{dis}(a)\) be the largest size of a strongly dissimilar subset of~\({\cal S}(a).\) Any two data points with indices in~\({\cal S}(a)\) are \emph{not} linked with probability at most~\(1-q_{low}\) and so we get for~\(k \geq 1\) that
\[\mathbb{P}\left(T_{dis}(a) \geq k \mid E_a\right) \leq (1-q_{low})^{k \choose 2} \leq \exp\left(-q_{low} {k \choose 2}\right) \leq \exp\left(-\frac{q_{low}k^2}{4}\right).\] Setting~\(k = \frac{3\log{n}}{q_{low}},\) we get that \[\mathbb{P}\left(T_{dis}(a) \geq k \mid E_a\right) \leq \exp\left(-4k\log{n}\right) \leq \frac{1}{n^{3}}. \] Combining this with~(\ref{sa_est}), we get that
\[\mathbb{P}\left(E_{a,st}\right) = 1-o(1),\]
where~\(E_{a, st} := \left\{T_{dis}(a) \leq \frac{3\log{n}}{q_{low}} \right\} \bigcap E_a.\)

If~\(E_{a,st}\) occurs, then the dataset~\(\{A_v\}\) contains at least~\(n_{low} := n(1-\varepsilon)^2 \zeta_{low}\) data points similar to~\(a\) (i.e.,~\({\cal S}(a)\) has at least~\(n_{low}\) indices) and the largest size of a strongly dissimilar subset of~\({\cal S}(a)\) is at most~\(\frac{3\log{n}}{q_{low}}.\) Any SD decomposition of the dataset~\(\{A_v\}\) is also an SD decomposition of the subset~\(\{A_v\}_{v \in {\cal S}(a)}\) and must therefore have  at least~\[\frac{n_{low}q_{low}}{3\log{n}} \geq \frac{(1-\varepsilon)^2\Sigma_{low}}{6\log{n}}\] batches, since~\(\Sigma_{low} = n\zeta_{low}q_{low} + v_0p_{low}s_{low}\) by definition and~\(n\zeta_{low}q_{low} \geq v_0p_{low}s_{low}\) by choice. Choosing~\(\varepsilon > 0\) small enough, this obtains the desired lower bound for~\(\tau_{SD}\) in~(\ref{sfd_low}) for the case~\(n\zeta_{low}q_{low} \geq v_0p_{low}s_{low}.\)

If~\(v_0p_{low}s_{low} \geq n\zeta_{low}q_{low}\) and~\(v_0p_{low}s_{low} \gg \log{n},\) then an analogous analysis as above using anomalous data points again obtains~(\ref{sfd_low}) and this completes the proof of part~\((b)\) of the Theorem.

Finally, suppose~\(v_0 = o(v_{low})\) so that~\(v_0 = o(n\zeta_{up}q_{up}^2)\) and also suppose that~\[n\zeta_{up}q_{up} \geq n\zeta_{low}q_{low} \gg \log{n}.\] This implies that~\(v_0p_{up}s_{up} = o(n\zeta_{up}q_{up})\) and therefore that~\(\Sigma = n\zeta_{up}q_{up}(1+o(1)),\) by the definition of~\(\Sigma\) in~(\ref{sigma_def}). Similarly, we also have that~\(v_0p_{up} = o(n\zeta_{up})\) and  therefore that~\(\Lambda = n\zeta_{up}s_{up}(1+o(1)),\) again by the definition of~\(\Lambda\) in~(\ref{sigma_def}). Consequently, we also have that~\(v_0\Lambda = o(\Sigma^2).\) Plugging all these into the upper bound~(\ref{sfd_up}), we obtain the desired upper bound for~\(\tau_{SD}\) in~(\ref{sfd_sub}).

The  lower bound in~(\ref{sfd_sub}) follows directly from~(\ref{sfd_low}) and the fact that~\(v_0 = o\left(\frac{n\zeta_{low}q_{low}}{p_{low}s_{low}}\right)\) so that~\(v_0p_{low}s_{low} = o(n\zeta_{low}q_{low})\) and therefore that~\(\Sigma_{low} = n\zeta_{low}q_{low}(1+o(1)).\)

The  lower bound in~(\ref{sfd_sup}) follows directly from~(\ref{sfd_low}) by an analogous argument as above. For the upper bound in~(\ref{sfd_sup}), we argue as follows. Since~\(v_0p_{up} \gg n\zeta_{up}\) and~\(v_0p_{up}s_{up} \geq v_0p_{low}s_{low} \gg \log{n},\) we get  from the definition of~\(\Sigma\) in~(\ref{sigma_def}), that~\(\Sigma = v_0p_{up}s_{up}(1+o(1)).\) In addition, since~\(s_{up} \geq p_{up},\) we also have that~\(v_0p_{up}s_{up} \gg n\zeta_{up}q_{up}\) and therefore that~\(\Lambda = v_0p_{up}s_{up}(1+o(1)),\) again  from the definition of~\(\Lambda\) in~(\ref{sigma_def}). Thus \(\sqrt{v_0 \Lambda} = v_0 \sqrt{p_{up}s_{up}}(1+o(1))\) is at least of the order of~\(\Sigma.\) This in turn implies that~\((\Sigma + \sqrt{v_0 \Lambda})\log{n}\) is much larger than~\(\Lambda\) and so the upper bound for~\(\tau_{SD}\) in~(\ref{sfd_sup}) follows directly from~(\ref{sfd_up}).~\(\qed\)

\emph{Proof of Corollary~\ref{cor_one}}: We begin by evaluating the main data point similarity probabilities~\(\zeta_{low}\) and~\(\zeta_{up}\) defined in~(\ref{zeta_def}). For~\(a \in \Omega,\) let~\({\cal N}(a) \subset \Omega\) be the set of all elements similar to~\(a.\) Any element~\(a\) is similar to exactly~\(r\) elements of~\(\Omega\) including itself and so by the uniform distribution of the main data points, we see that
\begin{equation}\label{zeta_up_eval}
\zeta_{up} = \zeta_{low} = \mathbb{P}\left(\rho(A_1,a) > \lambda\right) = \sum_{w \in {\cal N}(a)} \mathbb{P}(A_1=w)= \frac{r}{L}.
\end{equation}
From~(\ref{param_choice}) we know that~\(b \gg \frac{r}{L}\) and  that
\begin{equation}\label{zeta_up_ax}
\zeta_{up}  = \zeta_{low} = \frac{r}{L}\gg \frac{\log{n}}{n}.
\end{equation}

Next, we evaluate the anomalous similarity probabilities~\(p_{up}\) and~\(p_{low}\) defined in~(\ref{p_def}). Recalling that~\({\cal N}(a) \subset \Omega\) is the set of all elements similar to~\(a \in \Omega,\) we consider two cases depending on whether the element~\(a_{out} \in {\cal N}(a)\) or not. If~\(a_{out} \notin {\cal N}(a),\) then recalling that~\(A_n\) is an anomalous data point,   we get that
\begin{equation}\label{husn_tera}
\mathbb{P}\left(\rho(A_n,a) > \lambda\right) = \sum_{w \in {\cal N}(a)} \mathbb{P}(A_n = w) = \frac{r(1-b)}{L}.
\end{equation}
On the other hand, if~\(a_{out} \in {\cal N}(a),\) then
\begin{equation}\label{husn_tera_2}
\mathbb{P}\left(\rho(A_n,a) > \lambda\right) = b + \frac{r(1-b)}{L}
\end{equation}
and so combining, we get that
\begin{equation}\label{p_up_eval}
p_{low} = p_{up} = b +  \frac{r(1-b)}{L}
\end{equation}

The statement~\(v_0 = o\left(\frac{nrq^2}{L}\right)\) in Corollary statement implies that~\(v_0 = o(n\zeta_{up}q_{up}^2)\) (see~(\ref{link_example})) and so the bounds in~(\ref{sfd_sub_example}) follow directly from~(\ref{sfd_sub}). Similarly using the condition~\(o(1) = b \gg \frac{r}{L}\) in~(\ref{param_choice}), we see that~\(p_{up} = p_{low} = b(1+o(1)).\) The conditions~\(n \gg v_0 \gg \frac{nr}{bL}\) and~\(\frac{r}{L}\) in Corollary statement again ensure that  the bounds in~(\ref{sfd_sup}), Theorem~\ref{thm_one}\((c)\) are valid. This obtains~(\ref{sfd_sup_example}) and therefore completes the proof of the Corollary.~\(\qed\)

\renewcommand{\theequation}{\thesection.\arabic{equation}}
\setcounter{equation}{0}
\section{Proof of Theorem~\ref{thm_samp}\((a)\)} \label{sec_pf_thm_samp}
%CHNG HEEE TO INCLUDE ANAMOLY FREE ETC!!!!
%If~\({\cal L}(t)\) is the set of all~\(t-\)tuples with distinct entries; i.e., the set of all tuples of the form~\(({i_1},\ldots,{i_t})\) where~\(1 \leq i_1 \neq i_2 \ldots \neq i_t \leq n,\) then by construction~\((Z_1,\ldots,Z_t) \in {\cal L}(t).\)
As in Section~\ref{sec_pf_thm_one}, let~\(\Gamma  = \Gamma(\lambda) \subset K_n\) be the similarity graph obtained by connecting any two vertices~\(u\) and~\(v\)  by an edge if and only if the corresponding data points~\(A_u\) and~\(A_v\) are similar; i.e.,~\(\rho(A_u, A_v) > \lambda.\) Let~\(\Gamma_Z\) be the induced subgraph of~\(\Gamma\) with vertex set~\(\{Z_w\}_{1 \leq w \leq t}.\)

Our strategy is to use the first moment method to prove  that if~\(t\) is much smaller than~\(t_{low},\) then with high probability, no edge of~\(\Gamma\) is present in~\(\Gamma_Z.\) In what follows, we first derive a couple of auxiliary Lemmas used in the proof of Theorem~\ref{thm_samp}\((a)\) and then prove Theorem~\ref{thm_samp}\((a).\)

We begin by obtaining an upper bound for~\(N_{edge},\) the total number of edges in~\(\Gamma.\)
\begin{lemma} \label{lem_n_edge_a}  Suppose~\(v_0 = o(v_{samp})\) as in the statement of Theorem~\ref{thm_samp} and for constant~\(M > 0\) let~\(\Sigma = \Sigma(n,M)\) and~\(\Lambda = \Lambda(n,M)\) be as defined in~(\ref{sigma_def}). For every~\(0 < \varepsilon < 1,\) there are constants~\( M, K >0\)  large enough such that
\begin{equation}\label{e_edge_est_sub}
\mathbb{P}\left(E^c_{edge}\right) \leq \varepsilon,
\end{equation}
where~\(E_{edge} = E_{edge}(K,M) := \left\{N_{edge} \leq K n \Sigma\right\}.\)
\end{lemma}

\emph{Proof of Lemma~\ref{lem_n_edge_a}}:  From the discussion prior to Lemma~\ref{lemma_max_ave_deg}, we know that~\(d(u),\)  the number of data points similar to~\(A_u,\) is also the degree of~\(u\) in~\(\Gamma\) and so again using the handshaking relation, we have that
\begin{equation}\label{n_edge_ax_2}
N_{edge} = \frac{1}{2}\sum_{u=1}^{n}d(u).
\end{equation}
For~\(\varepsilon,M>0\) we recall the event~\(E_{sim} = E_{sim}(\varepsilon, M)\) defined in Lemma~\ref{lem_e_main} get from the respective probability estimate~(\ref{e_sim_est}) that if~\(M = M(\varepsilon) > 0\) is large enough, then~\(E_{sim}\) occurs with high probability.

We assume henceforth that~\(E_{sim}\) occurs so that each main data point~\(A_v, 1 \leq v \leq n-v_0,\) is similar to~\(d(v) \leq (1+\varepsilon)\Sigma\) other data points and each anomalous data point~\(A_v, n-v_0+1 \leq v \leq n,\) is similar to~\(d(v) \leq (1+\varepsilon)\Lambda\) other data points. From~(\ref{n_edge_ax_2}) we therefore get that
\begin{align}
N_{edge} \ind(E_{sim}) &\leq (1+\varepsilon)\frac{(n-v_0)\Sigma}{2} + (1+\varepsilon) \frac{v_0 \Lambda}{2} \nonumber\\
&\leq (1+\varepsilon)\frac{n\Sigma}{2} + (1+\varepsilon) \frac{v_0 \Lambda}{2}, \nonumber
\end{align}
where~\(\ind(.)\) refers to the indicator function. From~(\ref{sigma_def}), we know that
\[\Sigma  =  n\zeta_{up}q_{up} + v_0p_{up}s_{up} + M\log{n} \;\;\text{ and }\;\; \Lambda :=  (n\zeta_{up} + v_0p_{up})s_{up} + M\log{n}\] and so the condition~\(v_0 =o(v_{samp}) = o(n), v_{samp} = n\zeta_{up}q^2_{up},\) in Lemma statement ensures that
\begin{equation}\label{sigma_est}
\Sigma = n\zeta_{up}q_{up}(1+o(1)) + M\log{n}\;\;\text{ and }\;\;\Lambda = n\zeta_{up}s_{up}(1+o(1)) + M\log{n}.
\end{equation} Consequently,  we get that~\(v_0\Lambda = o(n\Sigma)\) and so
\begin{equation}\label{paramax_sundari}
N_{edge}\ind(E_{sim}) \leq (1+\varepsilon) n\Sigma \leq 2n\Sigma,
\end{equation}
for all~\(n\) large.

If~\(E^c_{sim}\) occurs, then we use the direct upper bound~\(N_{edge} \leq n^2\) and get from~(\ref{paramax_sundari}) and~(\ref{e_sim_est}) that
\[\mathbb{E}N_{edge} \leq 2 n\Sigma + n^2 \cdot \frac{2}{n^3} \leq 3 n\Sigma \] for all~\(n\) large, since~\(\Sigma\) is at least of the order of~\(\log{n},\) by definition. Applying the Markov inequality, we get that if~\(K = K(\varepsilon) > 0\) is large, then~(\ref{e_edge_est_sub}) holds. This completes the proof  of the Lemma.~\(\qed\)

Our next Lemma  estimates the number of edges present in the induced subgraph~\(\Gamma_Z\) of~\(\Gamma,\) for a deterministic realization~\(\omega \in E_{edge}.\) Formally,  let~\({\cal E}(\Gamma)\)  be the set of all edges in~\(\Gamma\) and define
\begin{equation}\label{nz_def}
N_Z := \sum_{f \in {\cal E}(\Gamma)} \ind(J_f),
\end{equation}
to be the number of edges in~\(\Gamma_Z,\) where~\(J_f\) denotes the event that the edge~\(f\) of~\(\Gamma\) is present in~\(\Gamma_Z.\) Letting~\(\mathbb{P}_{\omega}(.) = \mathbb{P}\left(. \mid  \omega\right)\) be the distribution conditioned on the realization~\(\omega,\) we have the following result.
\begin{lemma}\label{lemma_nz_a} Suppose the conditions in Lemma~\ref{lem_n_edge_a} hold and let~\(\Sigma = \Sigma(n,M)\) be as in~(\ref{sigma_def}). For constant~\(K > 0,\) let~\(E_{edge} = E_{edge}(K,M)\) be the event that the number of edges in~\(\Gamma\) is at most~\(Kn \Sigma,\) as defined in Lemma~\ref{lem_n_edge_a}. If~\(\omega \in E_{edge},\)  then
\begin{equation}\label{n_edge_ax_22}
\mathbb{E}_{\omega} N_{Z} \leq \frac{2Kt^2\Sigma}{n}.
\end{equation}
\end{lemma}
For completeness, we remark that the above bound holds for \emph{any}~\(K\) and~\(M\) and in our proof of Theorem~\ref{thm_samp}\((a)\) later, we choose~\(K\) and~\(M\) large enough so that the high probability estimates derived in Lemma~\ref{lem_n_edge_a} are valid.

In what follows,~\(o(1)\) represents a deterministic sequence~\(a_n,\) not depending on the choice of~\(\omega,\) that satisfies~\(a_n \rightarrow 0\) as~\(n \rightarrow \infty.\)\\
\emph{Proof of Lemma~\ref{lemma_nz_a}}: For any integer~\(1 \leq r \leq t\) and integers~\[1 \leq i_1 < i_2 < \ldots i_r \leq t \text{ and } 1 \leq l_1 \neq l_2 \ldots \neq l_r \leq n\] we have
\begin{equation}\label{kannamoochi}
\mathbb{P}_{\omega}\left(Z_{i_1} = A_{l_1},\ldots, Z_{i_r} = A_{l_r}\right) = \frac{1}{n(n-1)\ldots (n-r+1)}
\end{equation}
and so for any~\(1 \leq k_1 \neq k_2 \leq n\) we get
\begin{equation}\label{temu_ax}
\mathbb{P}_{\omega}\left(\{A_{k_1}, A_{k_2}\} \in \{Z_w\}_{1 \leq w \leq t} \right)  = \frac{t(t-1)}{n(n-1)}.
\end{equation}

Consequently, we obtain from~(\ref{temu_ax})
\begin{align}
\mathbb{E}_{\omega} N_{Z} &= \frac{t(t-1)}{n(n-1)}N_{edge} \nonumber\\
&\leq \frac{2t^2}{n^2} N_{edge} \nonumber\\
&\leq \frac{2t^2}{n^2} Kn\Sigma \nonumber\\
&= \frac{2Kt^2\Sigma}{n} \nonumber
\end{align}
for all~\(n\) large. This completes the proof of the Lemma.~\(\qed\)

\emph{Proof of Theorem~\ref{thm_samp}\((a)\)}: Given~\(\varepsilon > 0,\) let~\( M, K >0\) be large enough constants so that~(\ref{e_edge_est_sub}) holds; i.e., the event~\(E_{edge}\) occurs with probability at least~\(1-\varepsilon.\) Let~\(\omega \in E_{edge}\) so that the estimate~(\ref{n_edge_ax_22}) holds. From the estimate~(\ref{sigma_est}) for~\(\Sigma\) and the definition of~\(t_{low}\) in Theorem~\ref{thm_samp}\((a)\) statement, we see that~\[\frac{t^2\Sigma}{n} \longrightarrow 0 \text{ if and only if }t = o(t_{low}).\] Thus~\(\mathbb{E}_{\omega}N_Z = o(1)\) and an application of the Markov inequality obtains that~\[\mathbb{P}_{\omega}(N_Z \geq 1) \leq \mathbb{E}_{\omega}N_Z \leq \varepsilon\] for all~\(n\) large, not depending on~\(\omega \in E_{edge}.\)

Consequently
\begin{align}
\mathbb{P}\left(E_{edge} \bigcap \{N_Z \geq 1\}\right) &= \sum_{\omega \in E_{edge}} \mathbb{P}_{\omega}\left(N_Z\geq 1\right) \mathbb{P}(\Gamma=\omega) \nonumber\\
&\leq \varepsilon \sum_{\omega \in E_{edge}}\mathbb{P}(\Gamma = \omega) \nonumber\\
&\leq \varepsilon \nonumber
\end{align}
and  combining this with the estimate~(\ref{e_edge_est_sub}) in Lemma~\ref{lem_n_edge_a}, we get that
\begin{align}
\mathbb{P}(N_Z \geq 1) &= \mathbb{P}\left(E_{edge} \bigcap \left\{N_Z \geq 1\right\}\right) + \mathbb{P}\left(E^c_{edge} \bigcap \left\{N_Z \geq 1 \right\}\right) \nonumber\\
&\leq \varepsilon + \mathbb{P}\left(E^c_{edge} \bigcap \left\{N_Z \geq 1 \right\}\right) \nonumber\\
&\leq \varepsilon + \mathbb{P}(E^c_{edge}) \nonumber\\
&\leq 2\varepsilon \nonumber
\end{align}
for all~\(n\) large. In other words, with probability at least~\(1-2\varepsilon,\) the RUS dataset~\(\{A_{Z_w}\}_{1 \leq w \leq t}\) contains no edge of~\(\Gamma\) and therefore is SD. Since~\(\varepsilon > 0\) is arbitrary, this completes the proof of part~\((a)\) of the Theorem.~\(\qed\)

\renewcommand{\theequation}{\thesection.\arabic{equation}}
\setcounter{equation}{0}
\section{Proof of Theorem~\ref{thm_samp}\((b)\) and Corollary~\ref{cor_two}} \label{sec_pf_thm_samp_b}
We follow a similar strategy and structure as in the previous section: We aim to show that if~\(t\) is large enough to satisfy the conditions in the Theorem~\ref{thm_samp}\((b),\) then with high probability, at least one edge of the similarity graph~\(\Gamma\) is present in the induced subgraph~\(\Gamma_Z\) with vertex set~\(\{Z_j\}_{1 \leq j \leq t}.\) The main difference  is that we use the \emph{second moment method} and variance estimates to demonstrate that the number of edges~\(N_Z\) in~\(\Gamma_Z\) is close to its expected value, for nearly all realizations.  As before, we  present auxiliary Lemmas used in the proof of Theorem~\ref{thm_samp}\((b)\) and then prove Theorem~\ref{thm_samp}\((b).\)

Let~\(v_{samp} = n\zeta_{up}q^2_{up}\) be as in the statement of Theorem~\ref{thm_samp}. Recalling that~\(N_{edge}\) is the number of edges in~\(\Gamma\) and the event~\(E_{sim}\) defined in~(\ref{e_sim_def}), we have the following result.
\begin{lemma}\label{lem_n_edge_b} Let~\(\zeta_{low},\zeta_{up}\) be the main data point extremal similarity probabilities as defined in~(\ref{zeta_def}) and let~\(q_{low},q_{up}\) be the main data point linkage probabilities as defined in~(\ref{link_def})-(\ref{link_def_low}). If~\(v_0 = o(v_{samp})\) and~\(n\zeta_{low}q_{low} \gg \log{n},\) then there is a constant~\(M > 0\)  large enough such that
\begin{equation}\label{n_edge_low}
\mathbb{P}\left(E_{comb}\right) \geq 1-\frac{1}{n^3} = 1-o(1),
\end{equation}
where \[E_{comb}  = E_{comb}(M) := E_{sim} \bigcap \left\{\frac{n^2\zeta_{low}q_{low}}{M} \leq N_{edge} \leq Mn^2\zeta_{up}q_{up}\right\}.\]
\end{lemma}
\emph{Proof of Lemma~\ref{lem_n_edge_b}}: For the upper deviation bound for~\(N_{edge},\) we invoke the estimates derived in Lemma~\ref{lem_n_edge_a} of the previous section. Indeed, letting~\(\Sigma = \Sigma(n,M)\) be as defined in~(\ref{sigma_def}) we know from the estimate~(\ref{e_sim_est}) that if~\(M\) is large enough, then~\(E_{sim}\) occurs with  probability
\begin{equation}\label{e_sim_est_22}
\mathbb{P}(E_{sim}) \geq 1-\frac{1}{n^3}.
\end{equation} Moreover, we see from~(\ref{paramax_sundari}) that if~\(E_{sim}\) occurs, then the number of edges~\(N_{edge}\) in~\(\Gamma\) is~\(O(n\Sigma).\) Since~\(n\zeta_{up}q_{up} \geq n\zeta_{low}q_{low} \gg \log{n}\) and~\(v_0 = o(v_{samp}),\) by Lemma statement, we get from the definitions of~\(\Sigma\) and~\(\Lambda\) in~(\ref{sigma_def}) that
\begin{equation}\label{sigma_est_2}
\Sigma = n\zeta_{up}q_{up}(1+o(1)) \text{ and } \Lambda = n\zeta_{up}s_{up}(1+o(1))
\end{equation}
and this obtains the upper deviation bound in~(\ref{n_edge_low}).

For the lower bound on~\(N_{edge},\) we recall that~\(d(u),\)  the number of data points strongly similar to the data point~\(A_u,\) is also the degree of vertex~\(u\) in~\(\Gamma.\) Each main data point~\(A_v, 1 \leq v \leq n-v_0,\) is strongly similar to any other main data point with probability at least~\(\zeta_{low}q_{low}\) (see~(\ref{zeta_def}) and~(\ref{link_def_low})).  Therefore  given~\(A_u = a,\) we see that~\(d(u), 1\leq u \leq n-v_0\) is stochastically dominated from below by a Binomial random variable with parameters~\(n-1-v_0\) and~\(\zeta_{low}.\)

From Lemma statement, we know that~\(v_0 = o(v_{samp})= o(n)\) and  so using the deviation estimate~(\ref{conc_est_f}), we get for~\(\varepsilon > 0\) that
\[\mathbb{P}\left(d(u) \leq (1-\varepsilon)n\zeta_{low}q_{low} \mid A_u = a\right) \leq e^{-Cn\zeta_{low}q_{low}} \] for some constant~\(C > 0\) not depending on~\(a.\)  Averaging over~\(a\) and noting from Lemma statement that~\(n\zeta_{low}q_{low} \gg \log{n},\) we  obtain
\[\mathbb{P}\left(d(u) \leq (1-\varepsilon)n\zeta_{low}q_{low}\right) \leq e^{-Cn\zeta_{low}q_{low}} \leq \frac{1}{n^4},\] for all~\(n\) large. Setting~\[E_{low} := \bigcap_{u=1}^{n-v_0}\left\{d(u) \geq (1-\varepsilon)n\zeta_{low}q_{low}\right\},\] we then get by an application of the union bound that
\begin{equation}\nonumber
\mathbb{P}(E_{low}) \geq 1-\frac{1}{n^3}.
\end{equation}
Combining this with the estimate~(\ref{e_sim_est_22}) for~\(E_{sim},\) we get that
\begin{equation}\label{madai_ax}
\mathbb{P}(E_{sim} \cap E_{low}) \geq 1-\frac{2}{n^3} = 1-o(1).
\end{equation}

If~\(E_{sim} \cap E_{low}\) occurs, then each vertex~\(1 \leq v \leq n-v_0\) in the graph~\(\Gamma,\) has degree at least~\((1-\varepsilon)n\zeta_{low}q_{low}.\) The handshaking relation therefore  implies that
\[N_{edge} \geq (n-v_0)(1-\varepsilon)\frac{n\zeta_{low}q_{low}}{2} \geq (1-\varepsilon)^2 \frac{n\zeta_{low}q_{low}}{2}, \] since~\(v_0 = o(n).\) This obtains the desired the lower bound deviation for~\(N_{edge}\) in~(\ref{n_edge_low}) and therefore completes the proof of the Lemma.~\(\qed\)

In our next Lemma, we use the bounds for~\(N_{edge}\) derived in Lemma~\ref{lem_n_edge_b}, to estimate the number of edges~\(N_Z\) in the induced subgraph~\(\Gamma_Z,\) as defined in~(\ref{nz_def}).
\begin{lemma}\label{lemma_nz_b} Suppose the conditions in Lemma~\ref{lem_n_edge_b} hold and for constant~\(M > 0,\) let~\(E_{comb} = E_{comb}(M)\) be the event defined in Lemma~\ref{lem_n_edge_b}. If~\(\omega \in E_{comb},\)  then
\begin{equation}\label{n_edge_z_ax_22}
\frac{t^2 \zeta_{low}q_{low}}{2M} \leq\mathbb{E}_{\omega} N_{Z} \leq 2M t^2\zeta_{up}q_{up},
\end{equation}
where~\(\zeta_{low},\zeta_{up}\) are the extremal similarity probabilities as defined in~(\ref{zeta_def}) and~\(q_{low},q_{up}\) are the  linkage probabilities as defined in~(\ref{link_def})-(\ref{link_def_low}).
\end{lemma}

\emph{Proof of Lemma~\ref{lemma_nz_b}}:  From the relation~(\ref{temu_ax}) we know that any edge~\(f \in \Gamma\) is present in~\(\Gamma_Z\) with probability~\(\mathbb{P}_{\omega}(J_f) = \frac{t(t-1)}{n(n-1)}\) and so if~\(\omega \in E_{comb},\) then
\begin{align}
\mathbb{E}_{\omega} N_Z &= \sum_{f \in {\cal E}(\Gamma)} \mathbb{P}_{\omega}(J_f) \nonumber\\
&= \frac{t(t-1)}{n(n-1)}N_{edge} \nonumber\\
&\geq \frac{t^2}{n^2}N_{edge}(1+o(1)) \nonumber\\
&\geq \frac{t^2}{n^2}\frac{n^2\zeta_{low}}{M} (1+o(1)) \nonumber\\
&\geq \frac{t^2 \zeta_{low}}{2M} \label{enz_low}
\end{align}
for all~\(n\) large, not depending on the choice of~\(\omega,\) where the first inequality in~(\ref{enz_low}) is true since~\(t=o(n)\) by Theorem statement and the second inequality in~(\ref{enz_low}) follows from the lower bound for~\(N_{edge}\) in~(\ref{n_edge_low}).

A similar argument using the upper bound in~(\ref{n_edge_low}) also gives that~\(\mathbb{E}_{\omega} N_Z \leq 2M t^2 \zeta_{up}\)
for all~\(n\) large and this obtains an upper bound for the expected value of~\(N_Z.\) This completes the proof of the Lemma.~\(\qed\)

In our final Lemma,  we obtain a variance upper bound for~\(N_Z.\) Defining \[var_{\omega}(N_Z) := \mathbb{E}_{\omega} N^2_Z - \left(\mathbb{E}_{\omega} N_Z\right)^2\] to be the variance of~\(N_Z,\) we have the following result.
\begin{lemma}\label{lemma_var_nz} Suppose the conditions in Lemma~\ref{lem_n_edge_b} hold and for constant~\(M > 0,\) let~\(E_{comb} = E_{comb}(M)\) be the event defined in Lemma~\ref{lem_n_edge_b}. There is a constant~\(D > 0\) such that if~\(\omega \in E_{comb},\)  then
\begin{equation}\label{var_nz_est}
var_{\omega}(N_Z) \leq Dt^2\zeta_{up}q_{up} \left(1+  t\zeta_{up}q_{up}\right),
\end{equation}
where~\(\zeta_{up}\) and~\(q_{up}\) are the maximum similarity and linkage probabilities, as defined in~(\ref{zeta_def}) and~(\ref{link_def}), respectively.
\end{lemma}
We use the above variance estimate later in our proof of Theorem~\ref{thm_samp}\((b)\) along with the Chebychev inequality to show that~\(N_Z\) is close to its expected value with high~\(\mathbb{P}_{\omega}\) probability.

In what follows, we say~\(b_n= O(a_n)\) if there is a constant~\(C > 0,\)  not depending on the choice of~\(\omega,\) that satisfies~\(b_n \leq Ca_n\) for all~\(n \geq C.\)\\
\emph{Proof of Lemma~\ref{lemma_var_nz}}: Indeed, the variance of~\(N_Z\) is expressed in terms of the individual edge probabilities as
\begin{equation} \label{aaj_ki_raat}
var_{\omega}(N_Z)  = I_{sing} + I_{doub},
\end{equation}
where
\[ I_{sing} := \sum_{f \in {\cal E}(\Gamma)} I(f)\;\;\;\;\text{ with } \;\;\;\;\;I(f) := \mathbb{P}_{\omega}(J_f) - \mathbb{P}^2_{\omega}(J_f) \]
and~\[ I_{doub} := \sum_{f  \in {\cal E}(\Gamma)} \sum_{f_1 \in {\cal E}(\Gamma) \setminus \{f\}} I(f,f_1),\] with
\[I(f,f_1) := \mathbb{P}_{\omega}\left(J_{f} \cap J_{f_1}\right) - \mathbb{P}_{\omega}(J_{f})\mathbb{P}_{\omega}(J_{f_1}).\]

A direct estimate for~\(I(f)\) gives
\begin{align}
I_{sing} &= \sum_{f \in {\cal E}(\Gamma)} I(f) \nonumber\\
&\leq \sum_{f \in {\cal E}(\Gamma)} \mathbb{P}_{\omega}(J_f) \nonumber\\
&= \mathbb{E}_{\omega} N_Z  \nonumber\\
&\leq 2M t^2 \zeta_{up}q_{up}, \label{term_one_bound}
\end{align}
by the upper bound in~(\ref{n_edge_z_ax_22}).

To evaluate the term~\(I(f,f_1),\) we consider separately the cases that~\(f_1\) is vertex disjoint with~\(f\) or shares an endvertex with~\(f.\) In what follows, the edge~\(f= (y,z)\) has endvertices~\(y\) and~\(z.\)\\
\underline{\emph{Case I}}: We assume here that the edge~\(f_1\) does not share an endvertex with~\(f\) and so there are four endvertices in total in the set~\(\{f,f_1\}.\)  The number of ways of selecting these four vertices in the sequence~\((Z_1,\ldots,Z_t),\) is~\(t(t-1)(t-2)(t-4)\) and so the estimate~(\ref{kannamoochi}) gives us
\[\mathbb{P}_{\omega}(J_{f} \cap J_{f_1}) = \frac{t(t-1)(t-2)(t-4)}{n(n-1)(n-2)(n-4)}\]  and as before~\[\mathbb{P}_{\omega}(J_f) = \mathbb{P}_{\omega}(J_{f_1}) = \frac{t(t-1)}{n(n-1)} = \frac{t^2}{n^2}(1+o(1)).\] Thus
\begin{align}
I(f,f_1) &= \mathbb{P}_{\omega}\left(J_{f} \cap J_{f_1}\right) - \mathbb{P}_{\omega}(J_{f})\mathbb{P}_{\omega}(J_{f_1}) \nonumber\\
&= \frac{t(t-1)}{n(n-1)} \left(\frac{(t-2)(t-3)}{(n-2)(n-3)} - \frac{t(t-1)}{n(n-1)}\right) \nonumber\\
&= \frac{t^2}{n^2}\left(1+O\left(\frac{1}{t}\right)\right) \left(\frac{t^2}{n^2}\left(1+O\left(\frac{1}{t}\right)\right) - \frac{t^2}{n^2}\left(1+O\left(\frac{1}{t}\right)\right) \right) \nonumber\\
&= O\left(\frac{t^3}{n^4}\right). \label{tata_steel}
\end{align}

Because the realization~\(\omega \in E_{comb},\) there are at most~\(Mn^2\zeta_{up}q_{up}\) edges in~\(\Gamma,\) i.e., the size of the set~\({\cal E}(\Gamma)\) is at most~\(Mn^2\zeta_{up}q_{up}\) (see statement of Lemma~\ref{lem_n_edge_b}) and so  denoting~\({\cal D}(\Gamma,f)\) to be the set of all edges of~\(\Gamma\) vertex disjoint from~\(f,\) we get from~(\ref{tata_steel}) that
\begin{align}\label{bau_tits_three}
\sum_{f \in {\cal E}(\Gamma)} \sum_{f_1 \in {\cal D}(\Gamma,f)} &\leq \left(Mn^2 \zeta_{up}q_{up}\right)^2 \cdot O\left(\frac{t^3}{n^4}\right) \nonumber\\
&\leq D t^3 \zeta_{up}^2q_{up}^2,
\end{align}
for some constant~\(D > 0.\) This completes the analysis of Case~\(I.\)

It remains to consider the case where~\(f_1\) has an endvertex~\(z\) common with~\(f\) and we split the argument into two sub-cases, depending on whether~\(z\) is the index of a main data point or an anomaly.\\
\underline{\emph{Case II(A)}}: Suppose that the edge~\(f_1\) shares the endvertex~\(z\) with~\(f\) and~\(z\) is the index of a main data point, i.e.,~\(1 \leq z \leq n-v_0.\)

The number of ways of selecting the three vertices of~\(\{f,f_1\}\) in the sequence~\((Z_1,\ldots,Z_t),\) is~\(t(t-1)(t-2)\) and so from~(\ref{kannamoochi}), we get
\[\mathbb{P}_{\omega}(J_{f} \cap J_{f_1}) = \frac{t(t-1)(t-2)}{n(n-1)(n-2)} = \frac{t^3}{n^3}(1+o(1)) \leq \frac{2t^3}{n^3},\] since~\(t= o(n)\) by Theorem statement. Consequently
\begin{equation}\label{grumpo}
I(f,f_1) \leq \mathbb{P}_{\omega}\left(J_{f} \cap J_{f_1}\right) \leq \frac{2t^3}{n^3}.
\end{equation}

Because~\(\omega \in E_{sim},\) each main data point index is similar to at most~\((1+\varepsilon)\Sigma\) other data points and so from~(\ref{sigma_est_2}) we see that~\(z\) has degree~\[d(z) \leq (1+\varepsilon)\Sigma = (1+\varepsilon)n\zeta_{up}q_{up}(1+o(1)) \leq 2n\zeta_{up}q_{up}.\]  Thus at most~\(2n\zeta_{up}q_{up}\) edges of~\(\Gamma\) contain~\(z\) as an endvertex and so if~\({\cal H}(\Gamma, f) \subset {\cal E}(\Gamma) \setminus \{f\}\) is the set of all edges of~\(\Gamma\) sharing an endvertex with~\(f,\) then~\({\cal H}(\Gamma,f)\) has at most~\(4n\zeta_{up}q_{up}\) edges.

Similarly, if~\({\cal E}_{A}(\Gamma) \subset {\cal E}(\Gamma)\) is the set of all edges of~\(\Gamma\) containing at least one endvertex as the index of a main data point, then again using the fact that the sum of degrees is twice the number of edges, we get from the above discussion  that~\({\cal E}_A(\Gamma)\) has size at most~\(\frac{1}{2} \cdot n \cdot 4n\zeta_{up}q_{up} = 2 n^2\zeta_{up}q_{up}\) edges. Combining the above estimates for~\({\cal H}(\Gamma,f)\) and~\({\cal E}_A(\Gamma)\) with the estimate~(\ref{grumpo}) for~\(I(f,f_1),\) we get that
\begin{align}\label{bau_tits}
\sum_{f \in {\cal E}_{A}(\Gamma)}\sum_{f_1 \in {\cal H}(\Gamma, f)} I(f,f_1) &\leq 2n^2 \zeta_{up}q_{up} \cdot 4 n\zeta_{up}q_{up}\cdot \frac{2t^3}{n^3} \nonumber\\
&= 8 t^3\zeta_{up}^2q_{up}^2.
\end{align}
This completes the analysis of case~\(II(A).\)

\underline{\emph{Case II(B)}}: As in the previous case, the edge~\(f_1\) has~\(z\) as a common endvertex with~\(f,\) but here we assume that~\(z\) is the index of an anomaly, i.e.,~\(n-v_0+1 \leq z \leq n.\) The estimate~(\ref{grumpo}) derived above holds here as well and from the estimate for~\(\Lambda\) in~(\ref{sigma_est_2}), we see that the degree~\(d(z)\) satisfies
\begin{align}
d(z) &\leq (1+\varepsilon)\Lambda  \nonumber\\
&= (1+\varepsilon)n\zeta_{up}s_{up}(1+o(1)) \nonumber\\
&\leq 2n\zeta_{up}s_{up}, \label{glaxo}
\end{align}
since~\(\varepsilon < 1.\) Consequently, the set~\({\cal H}(\Gamma, f)\) of all edges of~\(\Gamma\) sharing an endvertex with~\(f,\) has size at most~\(2n\zeta_{up}s_{up}.\)

If~\({\cal E}_{B}(\Gamma) \subset {\cal E}(\Gamma)\) is set of all edges of~\(\Gamma\) containing at least one endvertex as the index of an anomaly data point, then using the fact that there are~\(v_0\) anomalies, the analysis in the previous paragraph implies that~\({\cal E}_{B}(\Gamma)\) has size at most~\[\frac{1}{2} \cdot v_0 \cdot 2n\zeta_{up}s_{up} = nv_0\zeta_{up}s_{up},\] again by the handshaking argument that the sum of vertex degrees is twice the number of edges. Consequently, we get from the estimate~(\ref{grumpo}) for~\(I(f,f_1)\) that
\begin{align}\label{bau_tits_two}
\sum_{f \in {\cal E}_{B}(\Gamma)}\sum_{f_1 \in {\cal H}(\Gamma, f)} I(f,f_1) &\leq nv_0\zeta_{up}s_{up} \cdot 2n \zeta_{up}s_{up} \cdot \frac{2t^3}{n^3} \nonumber\\
&= 4 \frac{v_0t^3\zeta_{up}^2s_{up}^2}{n}.
\end{align}
This completes the analysis of case~\(II(B).\)

Combining the estimates~(\ref{bau_tits_three}),~(\ref{bau_tits}) and~(\ref{bau_tits_two}) obtained above, we get that the double summation term~\(I_{doub}\) defined in the variance expression~(\ref{aaj_ki_raat}) is upper bounded as
\begin{align}
I_{doub} &= \sum_{f \in {\cal E}(\Gamma)} \sum_{f_1 \in {\cal E}(\Gamma) \setminus \{f\}} I(f,f_1) \nonumber\\
&= O\left(t^3\zeta_{up}^2\left(q_{up}^2 + \frac{v_0s_{up}^2}{n}\right)\right) \nonumber\\
&= O\left(t^3\zeta_{up}^2q_{up}^2\right),  \label{t_kelavi}
\end{align}
since~\(v_0 = o(v_{samp}),\) by the statement of Theorem~\ref{thm_samp}. Substituting~(\ref{t_kelavi}) and  the estimate~(\ref{term_one_bound}) for the remaining term~\(I_{sing}\) into~(\ref{aaj_ki_raat}), we get~(\ref{var_nz_est}) and this completes the proof of the Lemma.~\(\qed\)

\emph{Proof of Theorem~\ref{thm_samp}\((b)\)}: Let~\(\omega \in E_{comb}\) be any realization. Combining the variance estimate~(\ref{var_nz_est}) with  the lower bound~(\ref{enz_low}) for~\(\mathbb{E}_{\omega}N_Z,\) we get
\begin{align}
var_{\omega}\left(\frac{N_Z}{\mathbb{E}_{\omega}N_Z}\right) &= D \cdot \frac{t^2\zeta_{up}q_{up}}{t^4 \zeta_{low}^2q^2_{low}} +  D \cdot \frac{t^3 \zeta^2_{up}q^2_{up}}{t^4 \zeta_{low}^2q^2_{low}} \nonumber\\
&=  O\left(\frac{t^3 \zeta^2_{up}q^2_{up}}{t^4 \zeta_{low}^2q^2_{low}}\right) \label{riley_reid2}\\
&= o(1)  \label{riley_reid}
\end{align}
for some constant~\(D> 0,\) where~(\ref{riley_reid2}) follows since~\(t \gg \frac{1}{\zeta_{up}q_{up}}\) by the first  condition of~(\ref{fairuza}) in the statement of Theorem~\ref{thm_samp}\((b)\) and~(\ref{riley_reid}) follows from the second condition in~(\ref{fairuza}).

Thus given~\(\eta > 0,\)  we get from~(\ref{riley_reid}) and the Chebychev inequality that
\[\mathbb{P}_{\omega}\left(N_Z \geq (1-\varepsilon) \mathbb{E}_{\omega}N_Z\right) \geq 1- \frac{1}{\varepsilon^2}var_{\omega}\left(\frac{N_Z}{\mathbb{E}_{\omega} N_Z}\right) \geq 1-\eta\]
for all~\(n\) large, not depending on the choice of~\(\omega.\) From~(\ref{enz_low}), we know that~\(\mathbb{E}_{\omega}N_Z\) is at least of the order of~\(t^2\zeta_{low}q_{low} \rightarrow \infty\) (see third condition in~(\ref{fairuza})).

Summarizing, we therefore get that  if~\(\omega \in E_{comb}\) then
\[\mathbb{P}_{\omega}\left(N_Z \geq 1\right) \geq 1-\eta\] and using the probability estimate~(\ref{n_edge_low}) for~\(E_{comb},\) we finally get that
\[\mathbb{P}(N_Z \geq 1) \geq (1-\eta)\left(1-\frac{2}{n^3}\right) \geq 1-2\eta\] for all~\(n\) large.  In other words, with probability at least~\(1-2\eta,\) the undersampled tuple~\((Z_1,\ldots,Z_t)\) is not SD and since~\(\eta > 0\) is arbitrary, this completes the proof of the Theorem.~\(\qed\)

\emph{Proof of Corollary~\ref{cor_two}}: Follows directly from Theorem~\ref{thm_samp} and the fact that~\(\zeta_{up} = \zeta_{low} = \frac{r}{L} \gg \frac{\log{n}}{n}\) (see~(\ref{zeta_up_ax}) in the proof of Corollary~\ref{cor_one}) and~\(q_{up} = q_{low}  = 1\) (see~(\ref{link_example})).~\(\qed\)

\subsection*{\em Data Availability Statement}
Data sharing not applicable to this article as no datasets were generated or analysed during the current study.

\subsection*{\em Acknowledgement}
I thank Professors Rahul Roy, Thomas Mountford, Federico Camia, C. R. Subramanian, A. Ganesh  and the referees for crucial comments that led to an improvement of the paper. I also thank IMSc and IISER Bhopal for my fellowships.

\subsection*{\em Conflict of Interest and Funding Statement}
I certify that there is no actual or potential conflict of interest in relation to this article. This work was supported by the Engineering and Physical Sciences Research Council [Grant Ref: EP/Y028732/1].


\begin{thebibliography}{99}
\bibitem{abbey} J. D. Abbey and M .G. Meloy. (2017).
\newblock{Attention by Design: Using Attention Checks to Detect Inattentive Respondents and Improve Data Quality}.
\newblock{\em Journal of Operations Management}, \textbf{53-56}, 63--70.

\bibitem{agga} C. C. Aggarwal. (2017).
\newblock{An Introduction to Outlier Analysis}.
\newblock{\em In: Outlier Analysis}, Springer, Cham.~\(https://doi.org/10.1007/978-3-319-47578-3\_1.\)

\bibitem{alon} N. Alon and J. Spencer. (2008).
\newblock{\em The Probabilistic Method}.
\newblock{Wiley Interscience}.

\bibitem{sanjeev} Ch. S. K. Dash, A. K. Behera, S. Dehuri and A. Ghosh. (2023).
\newblock{An outliers detection and elimination framework in classification task of data mining}.
\newblock{\em Decision Analytics Journal}, \text{6}, 100164.


\bibitem{ganesan_dam} G. Ganesan. (2019).
\newblock{Graph Extensions, Edit Number and Regular Graphs}.
\newblock{\em Discrete Applied Mathematics}, \textbf{258}, 269--275.


\bibitem{ganesan} G. Ganesan. (2024).
\newblock{Dissimilar Batch Decomposition of Random Datasets}.
\newblock{\em Accepted for publication in Sankhya A}.

\bibitem{ganesan2} G. Ganesan. (2025).
\newblock{Reduced Similarity Decompositions and Subsets of Random Categorical Datasets}.
\newblock{\em Accepted for publication in Sankhya A}.

\bibitem{gong} Y. Gong, G. Liu, Y. Xue, R. Li and L. Meng. (2023).
\newblock{A Survey on Dataset Quality in Machine Learning}.
\newblock{\em Information and Software Technology}, \textbf{162}, 1--12.


\bibitem{kuhn} M. Kuhn  and K. Johnson. (2013).
\newblock{\em Applied Predictive Modeling}.
\newblock{Springer}.

\bibitem{hadi} H. Pouransari C-L. Li, J-H. Chang,  P. K. Vasu and C. Koc, V. Shankar and O. Tuzel. (2024).
\newblock{Dataset decomposition: faster LLM training with variable sequence length curriculum}.
\newblock{\em Proceedings of the~\(38^{th}\) International Conference on Neural Information Processing Systems (NIPS 2024)}, 1139, 1--27.

\bibitem{hidi} L. P. Rivest and M. Hidiroglou. (2004).
\newblock{Outlier treatment for disaggregated estimates}.
\newblock{\em Proceedings of the Section on Survey Research Methods, American Statistical Association}, pp. 4248--4256.



\bibitem{sullivan} J. H. Sullivan, M. Warkentin, L. Wallace. (2021).
\newblock{So many ways for assessing outliers: What really works and does it matter?}
\newblock{\em Journal of Business Research}, \textbf{132}, 530--543.


\bibitem{taleb} I. Taleb, M. A. Serhani, R. Dssouli. (2018).
\newblock{Big Data Quality Assessment Model for Unstructured Data}.
\newblock{\em \(13^{th}\) International Conference on Innovations in Information Technology (IIT 2018)}.

\bibitem{vin} H. P. Vinutha, B. Poornima and B. M. Sagar. (2018).
\newblock{Detection of outliers using interquartile range technique from intrusion dataset}.
\newblock{\em Information and Decision Sciences},  pp. 511--518.

\bibitem{west} D. West. (2001).
\newblock{\em Introduction to Graph Theory}.
\newblock{ Prentice Hall}.
\end{thebibliography}
\end{document}